%% file: main.tex
\documentclass{article}

\PassOptionsToPackage{numbers, compress}{natbib}

 \usepackage[preprint]{nips_2018} %

\usepackage[utf8]{inputenc} 
\usepackage[T1]{fontenc}    
\usepackage{hyperref}       
\usepackage{url}            
\usepackage{booktabs}       
\usepackage{amsfonts}       
\usepackage{nicefrac}       
\usepackage{microtype}      
\usepackage{graphicx}
\usepackage{subfigure}
\usepackage{amssymb}
\usepackage{amsmath}
\usepackage[linesnumbered]{algorithm2e}
\usepackage{algpseudocode}
\usepackage{wrapfig}
\usepackage{lipsum}
\usepackage{rotating}
\usepackage{hyperref}
\usepackage{theorem}
\newcommand{\BlackBox}{\rule{1.5ex}{1.5ex}}  
\newenvironment{proof}{\par\noindent{\bf Proof\ }}{\hfill\BlackBox\\[2mm]}
\newtheorem{theorem}{Theorem}
\newtheorem{lemma}[theorem]{Lemma} 

\title{A Simple Saliency Method That Passes the Sanity Checks}

\author{
Arushi Gupta \\
\texttt{arushig@princeton.edu}\\
\And 
Sanjeev Arora\\
\texttt{arora@cs.princeton.edu}\\
}

\begin{document}

\raggedbottom

\maketitle

\input{abstract}

\input{introduction}

\input{CGIexplanation}



\input{param_randomization_test}





\input{conclusion}

\small
\bibliographystyle{unsrtnat}
\bibliography{main}

\newpage
\input{appendix}

\end{document}

%% file: abstract.tex
\begin{abstract}

There is great interest in {\em saliency methods} (also called {\em attribution methods}), which  give \textquotedblleft explanations\textquotedblright\ for a deep net's decision, by assigning a {\em score} to each feature/pixel in the input. Their design usually involves  credit-assignment via the gradient of the output with respect to input. 
Recently \citet{adebayosan} questioned the validity of many of these methods since they do not pass simple {\em sanity checks} which test whether the scores shift/vanish when  layers of the trained net are randomized, or when the net is retrained using random labels for inputs. 

We propose a simple fix to existing saliency methods that helps them pass sanity checks, which we call {\em competition for pixels}. This involves computing saliency maps for all possible labels in the classification task, and using a simple competition among them to identify and remove less relevant pixels from the map. The simplest variant of this is  {\bf Competitive Gradient \( \odot \) Input (CGI)}: it is efficient, requires no additional training, and uses only the input and gradient. Some theoretical justification is provided for it (especially for ReLU networks) and its performance is empirically demonstrated.

\end{abstract}

%% file: introduction.tex
\section{Introduction}

Methods that allow a human to \textquotedblleft understand\textquotedblright or \textquotedblleft interpret\textquotedblright\ the decisions of deep nets have become increasingly important as deep learning moves into applications ranging from self-driving cars to analysis of scientific data. For simplicity our exposition will assume the deep net is solving an  image classification task, though the discussion extends to other data types.  In such a case the explanation consists of assigning saliency scores (also called attribution scores)  to the pixels in the input, and presenting them as a heat map  to the human. 
 
Of course, the idea of \textquotedblleft credit assignment\textquotedblright\ is already embedded in gradient-based learning, so a natural place to look for saliency scores is  the gradient of the output with respect to the input pixels. Looking for high coordinates in the gradient is akin to classical {\em sensitivity analysis} but in practice does not yield high quality explanations. However, gradient-like notions are the basis of other more successful methods. {\bf Layer-wise Relevance Propagation (LRP)} \citet{bach2015pixel}  uses a back-propagation technique where every node in the deep net receives a share of the output which it distributes to nodes below it. This happens all the way to the input layer, whereby every pixel gets assigned a share of the output, which is its score. Another rule Deep-Lift  \citet{DBLP:journals/corr/ShrikumarGK17} does this in a different way and is related to Shapley values of cooperative game theory.  

The core of many such ideas is a simple map called {\bf Gradient \( \odot \) Input}  \citet{DBLP:journals/corr/ShrikumarGK17}~: the score of a pixel in this rule is product of its  value and the partial derivative of the output with respect to that pixel. Complicated methods often reduce to {\bf Gradient \( \odot \) Input} for simple ReLU nets with zero bias. See~\citet{montavon2018methods} for a survey.

Recently~\citet{adebayosan} questioned the validity of many of these  techniques by suggesting that they don't pass simple \textquotedblleft sanity checks.\textquotedblright\ \ Their checks involve randomizing the model parameters or the data labels (see Section~\ref{sec:relwork} for details). They find that maps produced using corrupted parameters and data are often difficult to visually distinguish from those produced using the original parameters and data. This ought to make the maps less useful to a human checker. 
The authors concluded that \textquotedblleft {\em ...widely deployed saliency methods are
independent of both the data the model was trained on, and the model parameters}.\textquotedblright

The current paper focuses on multiclass classification and introduces a simple modification to existing methods: {\em Competition for pixels}. Section~\ref{sec:CGI} motivates this by pointing out a significant issue with previous methods: they produce saliency maps for a chosen output (label) node using  gradient information only for that node while ignoring the gradient information from the other (non-chosen) outputs.  To incorporate information from non-chosen labels/nodes in the multiclass setting we rely on a property called {\em completeness}  used in earlier methods, according to which the sum of pixel scores in a map is equal to the value of the chosen node (see Section~\ref{sec:CGI}). One can design saliency maps for all outputs and use completeness to assign a pixel score in each map.   One can view the various scores assigned to a single pixel  as its \textquotedblleft votes\textquotedblright\ for different labels. The competition idea is roughly to zero out any pixel whose vote for the chosen label was lower than for another (non-chosen) label. Section~\ref{subsec:theory} develops theory to explain why this modification helps pass sanity checks in the multi-class setting, and yet produces maps not too different from existing saliency maps. Section ~\ref{alg:desc} gives the formal definition of the algorithm. 

Section~\ref{sec:paramran} describes implementation of this idea for two well-regarded methods, {\bf Gradient  \( \odot \) Input} and {\bf LRP} 

and shows that they produce sensible saliency maps while also passing the sanity checks. 
We suspect our modification can make many other methods pass the sanity checks.

\section{Past related work}

We first recall the sanity checks proposed in~\citet{adebayosan}. 
\label{sec:relwork}

\textbf{The model parameter randomization test.} According to the authors, this "compares the output of a saliency method on a trained model with the output of the saliency method on a randomly initialized untrained network of the same architecture." The saliency method fails the test if the maps are similar for trained models and randomized models. The randomization can be done in stages, or layer by layer. 

\textbf{The data randomization test} "compares a given saliency method applied to a model trained on a labeled data set with the method applied to the same model architecture but trained on a copy of the data set in which we randomly permuted all labels." Clearly the model in the second case has learnt no useful relationship between the data and the labels and does not generalize. The saliency method fails if the maps are similar in the two cases on test data.

\subsection{Some saliency methods}

Let $S_y$ denote the logit computed for the chosen output node of interest, $y$.

\begin{enumerate}
\item \textbf{The Gradient  \( \odot \) Input explanation}:
Gradient  \( \odot \)  Input method  \citet{DBLP:journals/corr/ShrikumarGK17} computes 
\(. \frac{\partial  S_y}{\partial x } \odot x \) where \( \odot \) is the elementwise product. 

\item \textbf{Integrated Gradients}
Integrated gradients \citet{sundararajan2017axiomatic} also computes the gradient of the chosen class's logit. However, instead of evaluating this gradient at one fixed data point, integrated gradients consider the path integral of this value as the input varies from a baseline, \( \bar{x}\), to the actual input, \(x\) along a straight line.

\item \textbf{Layerwise Relevance Propagation}
\citet{bach2015pixel} proposed an approach for propagating importance scores called Layerwise Relevance Propagation (LRP). LRP decomposes the output of the neural network into a sum of the relevances of coordinates of the input. Specifically, if a neural network computes a function \(f(x)\) they attempt to find relevance scores \(R_p^{(1)}\) such that \(  f(x)  \approx \sum_p R_{p}^{(1)} \)

\item  \textbf{Taylor decomposition}  As stated \citet{montavon2018methods}  for special classes of piecewise linear functions that satisfy \( f(tx) = tf(x) \), including ReLU networks with no biases, one can always find a root point near the origin such that 
\( f(x) = \sum_{i=1}^d R_i(x) \)
where the relevance scores \( R_i(x) \) simplify to \(  R_i(x) = \frac{\partial f}{\partial x_i}\cdot x_i  \)

\item \textbf{DeepLIFT explanation} The DeepLIFT explanation \citet{DBLP:journals/corr/ShrikumarGK17}  calculates the importance of the input by comparing each neuron's activation to some 'reference' activation. Each neuron is assigned an attribution that represents the amount of difference from the baseline that that neuron is responsible for. Reference activations are determined by propagating some reference input, \( \bar{x} \), through the neural network. 

\end{enumerate}

\textbf{Relationships between different methods}
.\citet{kindermans2016investigating} and \citet{DBLP:journals/corr/ShrikumarGK17} showed that if modifications for numerical stability are not taken into account, the LRP rules are equivalent within a scaling factor to Gradient \( \odot \) Input. 
\citet{ancona2017towards} showed that for ReLU networks (with zero baseline and no biases) the \( \epsilon \)-LRP and DeepLIFT (Rescale) explanation methods are equivalent to the Gradient \( \odot \) Input.

%% file: CGIexplanation.tex
\section{Adding competition}
\label{sec:CGI}

\begin{figure}
\centering
\includegraphics[scale =.15]{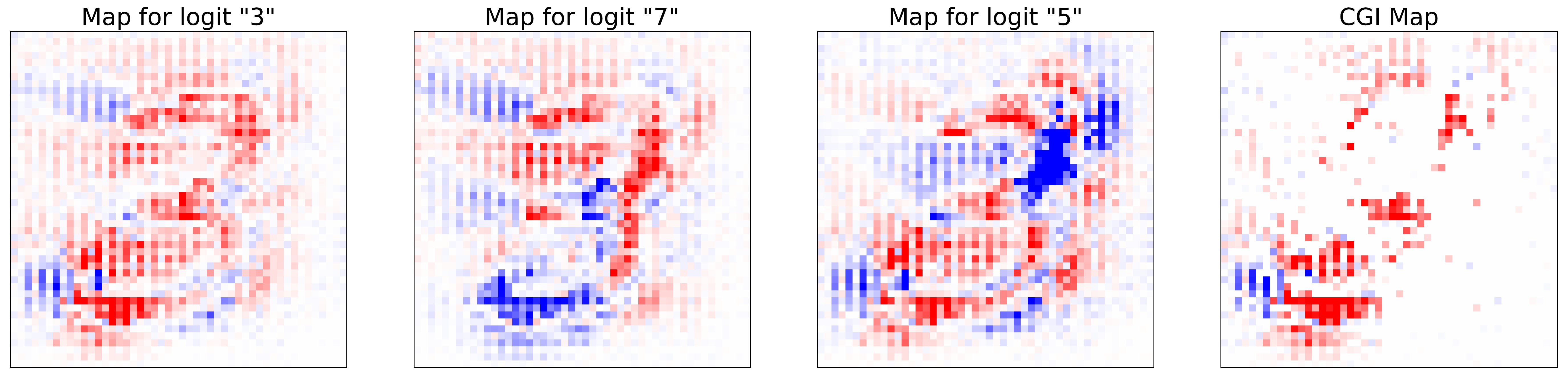}
\caption{Heatmap of {\bf Gradient \( \odot \)  Input} saliency maps produced by various logits of a deep net trained on MNIST.  Red denotes pixels with positive values  and Blue denotes negative values. The input image is of the number \( 3 \) , which is clearly visible in  all maps. Note how maps computed using logits/labels "\( 7" \) and " \( 5 \)" assign red color (resp., blue color) to pixels that would have been expected to be present (resp., absent) in those digits.  The last figure shows the map produced using our CGI method. }
\label{fig:logitmaps}
\end{figure}

The idea of competition suggests itself naturally when one examines saliency maps produced using {\em all} possible labels/logits in a multiclass problem, rather than just the chosen label.  Figure~\ref{fig:logitmaps} shows some {\bf Gradient$\odot$Input} maps produced using  AlexNet trained on MNIST \citet{lecun1998mnist}, where the first layer was modified to accept one color channel instead of 3.  Notice: {\em Many pixels found irrelevant by humans receive heat (i.e.\ positive value) in all the maps, and many relevant pixels receive heat in more than one map.}  Our experiments showed similar phenomenon on more complicated datasets such as ImageNet. This figure highlights an important point of \citet{adebayosan} which is that many saliency maps pick up a lot of information about the input itself ---e.g., presence of sharp edges--that are at best incidental to the  final classification. Furthermore, these incidental features can survive during the various randomization checks, leading to failure in the sanity check. Thus it is a natural idea to create a saliency map by {\em combining} information from all labels, in the process filtering out or downgrading the importance of incidental features.

Suppose the input is \( x \) and the net is solving a \( k \)-way classification. We assume a standard softmax output layer  whose inputs are \( k \) logits, one per label. Let $x$ be an input, $\ell$ be its label and \( y^{\ell}\) denote the corresponding logit. To explain the output of the net many methods assign a score to each pixel by using the gradient of $y^{\ell}$ with respect to $x$.
For concreteness, we use {\bf Gradient\( \odot\)Input} method, which  assigns score \( x_i f^{\ell}_i \) to pixel \( i \) where \( f^{\ell}_i \) os the coordinate in the gradient  corresponding to the \( i \)th pixel \(x_i\). 

Usually prior methods do not examine the logits corresponding to non-chosen labels as well, but as mentioned, we wish to ultimately design a simple competition among labels for pixels. {\em A priori} it can be unclear how to compare scores across labels, since this could end up being an \textquotedblleft apples vs oranges\textquotedblright\ comparison due to potentially different scaling. However, prior work \citet{sundararajan2017axiomatic} has identified a property called {\em completeness}:  this requires that the sum of the pixel scores is exactly the logit value. {\bf Gradient \( \odot \) Input} is an attractive method because it satisfies completeness exactly for ReLU nets with zero bias.  Recall that the ReLU function with {\em bias} \( a \) is \( ReLU(z, a)\max \{z-a, 0\}. \)  
\begin{lemma} {\em On ReLU nets with zero bias  {\bf Gradient \( \odot \)  Input} satisfies completeness.}
\end{lemma}
\begin{proof} 
If function \( f \)  is computed by a ReLU net with zero bias at each node, then it satisfies \( f(\lambda x) = \lambda f(x) \) . Now partial differentiation with respect to \( \lambda \) at \( \lambda =1 \)  gives \( x \cdot \nabla_x(f) = f(x) \) .
\end{proof} 

Past work shows how to design methods that satisfy completeness for ReLU with nonzero bias by computing integrals, which is more expensive (see~\cite{ancona2017towards}, which also explores interrelationships among methods). However, we find empirically this is not necessary because of the following phenomenon. 

\noindent{\bf Approximate completeness.}  For ReLU nets with nonzero bias, {\bf Gradient \( \odot \)  Input} in practice have the property that the sum of pixel scores varies fairly linearly with the logit value (though theory for this is lacking). See Figure \ref{fig:logitmaps} which plots this for VGG-19 trained on Imagenet. Thus up to a scaling factor, we can assume  {\bf Gradient \( \odot \)  Input} approximately satisfies completeness.

\paragraph{Enter competition.} Completeness (whether exact or approximate) allows us to consider the score of a pixel in {\bf Gradient \( \odot \) Input} as a \textquotedblleft vote\textquotedblright\ for a label. Now consider the case where $\ell$ is the label predicted by the net for input $x$. Suppose pixel $i$ has a positive score for label $\ell$ and an even more positive score for label $\ell_1$. This pixel contributes positively to both logit values. But remember that since label $\ell_1$ was not predicted by the net as the label, the logit $y^{\ell_1}$ is {\em less} than than logit $y^{\ell}$, so the contribution of pixel $x_i$'s \textquotedblleft vote\textquotedblright\ to $y^{\ell_1}$ is proportionately even higher than its contribution to $y^{\ell}$. This perhaps should make us realize that this pixel may be less relevant or even irrelevant to label $\ell$  since it is effectively siding with label $\ell_1$ (recall Figure~\ref{fig:logitmaps}).  We conclude that looking at {\bf Gradient $\odot$ Input} maps for non-chosen labels should allow us to fine-tune our estimate of the relevance of a pixel to the chosen label.

\begin{figure}
\centering

\includegraphics[scale=.3]{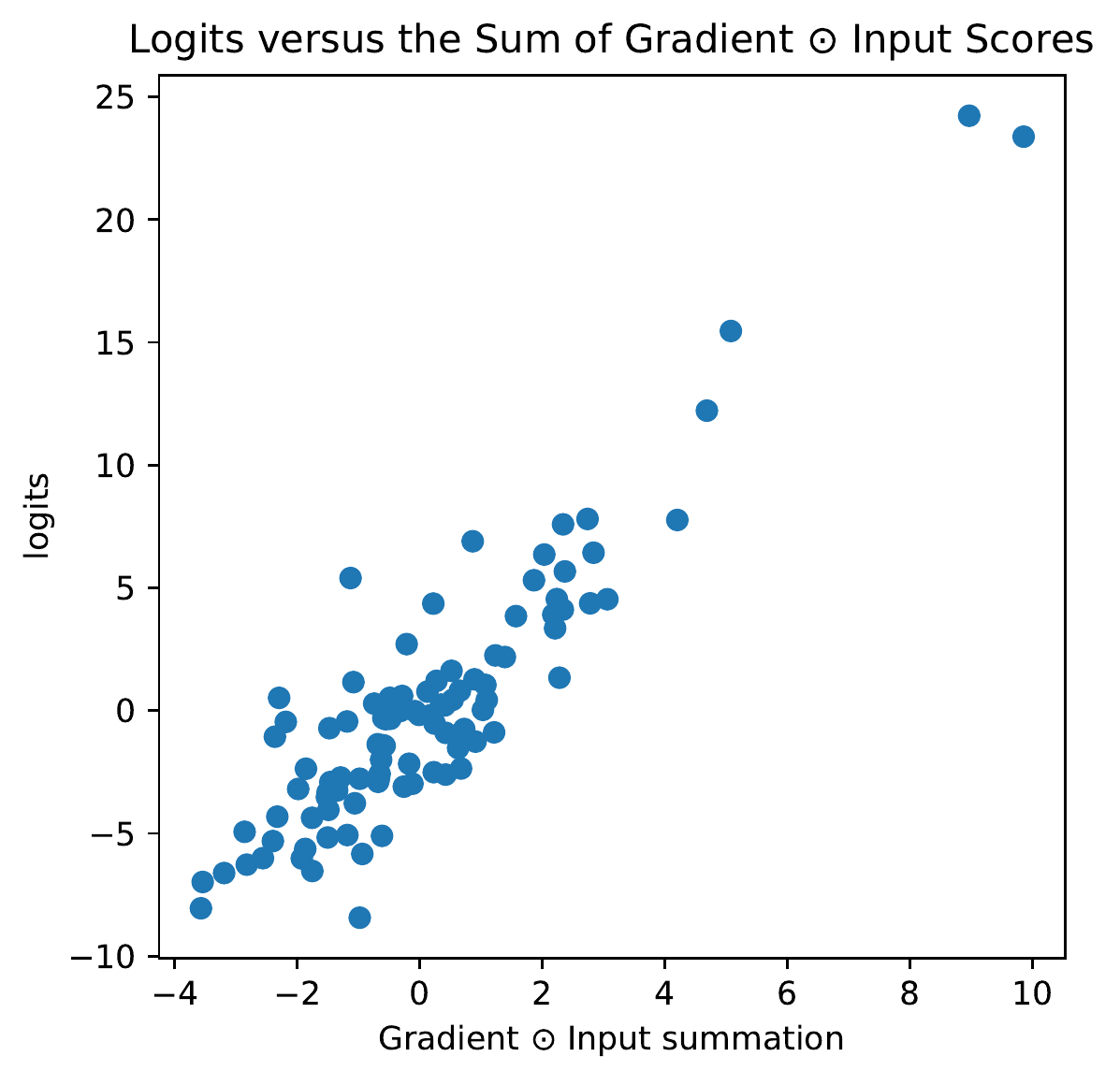}
\caption{Approximate completeness property of Gradient $\odot$ Input on ReLU nets with nonzero bias (VGG -19).  An approximately linear relationship holds between logit values  and the sum of the pixel scores for Gradient $\odot$ Input for a randomly selected image.}
\label{logitvgi}
\end{figure}

Now we formalize the competition idea. Note that positive and negative pixel scores should be interpreted differently; the former should be viewed as supporting the chosen label,  and the latter as opposing that label.

{\bf Competitive Gradient \( \odot \) Input~(CGI):} {\em Label \( \ell \) \textquotedblleft wins\textquotedblright a pixel if either (a) its map assigns that pixel as positive score higher than the scores assigned by every other label, or (b) its map assigns the pixel a negative score lower than the scores assigned by every other label.   The final saliency map consists of scores assigned by the chosen label \( \ell \) to each pixel it won, with the map containing a score \( 0 \) for any pixel it did not win.}

Using the same reasoning as above, one can add competition to any other saliency map that satisfies completeness. Below we also present experiments on adding competition to {\em LRP}. In Sections \ref{sec:paramran} and \ref{sec:dataran}  we present experiments showing that adding competition makes these saliency methods pass sanity checks.

\subsection{Why competition works: some theory}
\label{subsec:theory}

Figure~\ref{fig:logitmaps} suggests that it is a good idea to zero out some pixels in existing saliency maps. Here we develop a more principled understanding of why adding competition (a) is aggressive enough to zero out enough pixels to help pass sanity checks on randomized nets and (b) not too aggressive so as to retain a reasonable saliency map for properly trained nets.

\citet{adebayosan} used linear models to explain why methods like {\bf Gradient$\odot$Input} fail their randomization tests. These tests turn the gradient into a random vector, and if $\xi_1, \xi_2$ are random vectors, then $x\odot \xi_1$  and $x \odot \xi_2$
are visually quite similar when $x$ is an image. (See Figure 10 in their appendix.) Thus the saliency map retains a strong sense of $x$ after the randomization test, even though the gradient is essentially random. 
Now it is immediately clear that with $k$-way competition among the labels, the saliency map would be expected to become almost blank  in the randomization tests since each label is equally likely to give the highest score to a pixel so it becomes zero with probability $1-1/k$. Thus we would expect that adding competition enables the map to pass the sanity checks. In our experiments later we see that the final map is indeed very sparse.  

But one cannot use this naive model to understand why CGI does not also destroy the saliency map for properly trained nets.  The reason being that the gradient is not random and  depends on the input. In particular if $g_1$ is the gradient of a logit with respect to input $x$ then $g_1 \cdot x$ is simply the sum of the coordinates of $g_1\odot x$, which due to completeness property has to track the logit value. In other words, gradient and input are correlated, at least when logit is sufficiently nonzero. Furthermore, the amount of this correlation is given by the logit value. In practice we find that if the deep net is trained to high accuracy on a dataset, the logit corresponding to the chosen label is significantly higher than the other logits, say 2X or 4X. This higher correlation plays a role in why competition ends up preserving much of the information.

 We find the following model of the situation simplistic but illustrative: Assume gradient $g_1$ and input $x$ are random vectors drawn from  ${\mathcal N}(0, 1/n)^n$  {\em conditional} on $g_1\cdot x \geq \delta$ (i.e., correlated random vectors), where $\delta$ corresponds to the logit value. On real data we find that $\delta$ is, $0.1$ to $0.2$ for the chosen label, which is a fairly significant since the inner product of two independent draws from ${\mathcal N}(0, 1/n)^n$ would be only $1/\sqrt{n})$ in magnitude, say $0.01$ when  $n =10000$. 

Let $g_2$ be the gradient of a second (non-chosen) logit with respect to $x$. Figure~\ref{fig:logitmaps} suggests that actually $g_1$ and $g_2$ can have significant overlap in terms of their high coordinates, which we referred to earlier as {\em shared features} or {\em incidental features}  (see Figure~\ref{fig:logitmaps}). We want competition to give us a final saliency map that downplays  pixels in this overlap, though not completely eliminate them. 

Without loss of generality let the first $n/2$ coordinates correspond to the shared features. So we can think of $g_1 = (h_1, \xi_1)$ and $g_2 =( h_2, \xi_2)$ where 
$h_1, h_2$ respectively are the sub-vectors of $g_1, g_2$ respectively in the shared features and $\xi_1, \xi_2$ are random $n/2$-dimensional vectors in the second halves. All these vectors are assumed to be unit vectors. It is unreasonable to expect the coordinates of $h_1$ and  $h_2$ to be completely identical, but we assume there is significant correlation, so assume $h_1 \cdot h_2 \geq 1/2$. 

Now imagine picking the input $x$ as mentioned above: Given $g_1$ it is a random vector conditional on $g_1\cdot x \geq \delta$. Then a simple calculation via measure concentration shows that half of the this inner product of $\delta$ must come from the first $n/2$ coordinates, meaning $(h_1, 0) \cdot x \approx \delta /2$. Another application of measure concentration shows that $(h_2, 0)\cdot x \approx \delta/4$, reflecting the fact that $h_1 \cdot h_2 =1/2$.

{\em What happens after we apply competition (i.e., CGI)?} An exact calculation requires a multidimensional integral using the Gaussian distribution.  But simulations (see Figures ~\ref{fig:c1}, \ref{fig:c2} in appendix) show that
after zeroing out coordinates in $g_1 \odot x$ due to competition from $g_2\odot x$, we have a contribution of at least $c_1\delta/2$ left from the first $n/2$ coordinates and a contribution of at least $c_2 \delta/2$ from the last $n/2$ coordinates, where $c_1, c_2$ are some constants. In other words, there remains a significant contribution from both the shared features, and the non-shared features.  Thus the competition still allows the saliency map to retain some kind of approximate completeness.

\noindent{\em Remark 1:} There is something else missing in the above account which in practice ensures that competition is not too aggressive in zeroing out pixels in normal use: {\em entries in the gradients are non-uniform, so the subset of coordinates with high values is somewhat sparse.}   Thus for each label/logit, the bulk of its score is carried by a subset of pixels. If each label concentrates its scores on \(\rho\) fraction of pixels then heuristically one would expect two labels to compete only on \( \rho^2 \) fraction of pixels. For example if \( \rho = 0.2\) then they would compete only on \(0.04\) or \(4\)\% of the pixels. This effect can also be easily incorporated in the above explanation. {\em Remark 2:} The above analysis suggests that saliency map can make sense for any label with a sufficiently large logit (eg the logit for label "7" in Figure~\ref{fig:logitmaps}.)

\subsection{Formal Description of CGI}
\label{alg:desc}

Here we provide a formal definition of our algorithm, CGI, which can be found in Algorithm \ref{alg:CGI},

Let \( S[i] \) denote the logit computed by the \( ith \) output node of our neural network. For each output node,  \( i \in [1,...,C] \), and for each scalar coordinate of the input, \( x_j \) we compute \( \frac{\partial S[i]}{\partial x_j}\), i.e. we compute Gradient  \( \odot \) Input for each scalar element  $x_j$ of x  for each of the C output nodes. Letting y denote the index of the chosen label, if \( \frac{\partial S[y]}{\partial x_j} \cdot x_j>0\), $x_j$ will be included in the heat map if \( \frac{\partial S[y]}{\partial x_j} \cdot x_j \)  is equal to the maximum of \( \{\frac{\partial S[i]}{\partial x_j} \cdot x_j \}, i \in {1,...,C} \) , and its value in the heat map will be \( \frac{\partial S[y]}{\partial x_j} \cdot x_j\). If \( \frac{\partial S[y]}{\partial x_j } \cdot x_j <0 \), it will be included in the heat map if  \( \frac{\partial S[y]}{\partial x_j} \cdot x_j\) is equal to the minimum of \( \{\frac{\partial S[i]}{\partial x_j} \cdot x_j \}, i \in {1,...,C} \) , and its value in the heat map will be \( \frac{\partial S[y]}{\partial x_j} \cdot x_j \). For all other inputs, the default value in the heat map is 0.

\begin{algorithm}[H]
	\KwIn{An image \( \in R^{d} \) and a neural network \( S : R^{d} \rightarrow R^{C}\) }
	initialization: set H = \( 0 \in R^{d} \) vector. Let y be the index of the chosen output node\; 
	\For{Element in Image }{
	        Calculate \( \frac{\partial S[i]}{\partial \text{Element}} \) for all output nodes \( S[i] \)\\
		\eIf{ \( \frac{\partial S[y]}{\partial \text{Element}} \cdot \text{Element} > 0 \)}
		{
		
		\If{ \( \frac{\partial S[y]}{\partial \text{Element}} \cdot \text{Element} \geq \frac{\partial S[i]}{ \partial \text{Element}} \text{Element} \forall i \neq y \)}
		{
		Make the corresponding element of  \( H \) equal to \(\frac{\partial S[y]}{\partial \text{Element}} \cdot \) Element }
		
		}
		{
		
		\If{ \( \frac{\partial S[y]}{\partial \text{Element}} \cdot \text{Element} \leq \frac{\partial S[i]}{ \partial \text{Element}} \text{Element} \forall i \neq y \) }
		{
		Make the corresponding element of  \( H \) equal to \(\frac{\partial S[y]}{\partial \text{Element}} \cdot \) Element 
		}

		}
	}
	\caption{Competitive Gradient \( \odot \) Input}
	\label{alg:CGI}
\end{algorithm}

%% file: param_randomization_test.tex
\section{Experiments}
\label{sec:paramran}
\begin{figure}
\centering

\includegraphics[scale=.09]{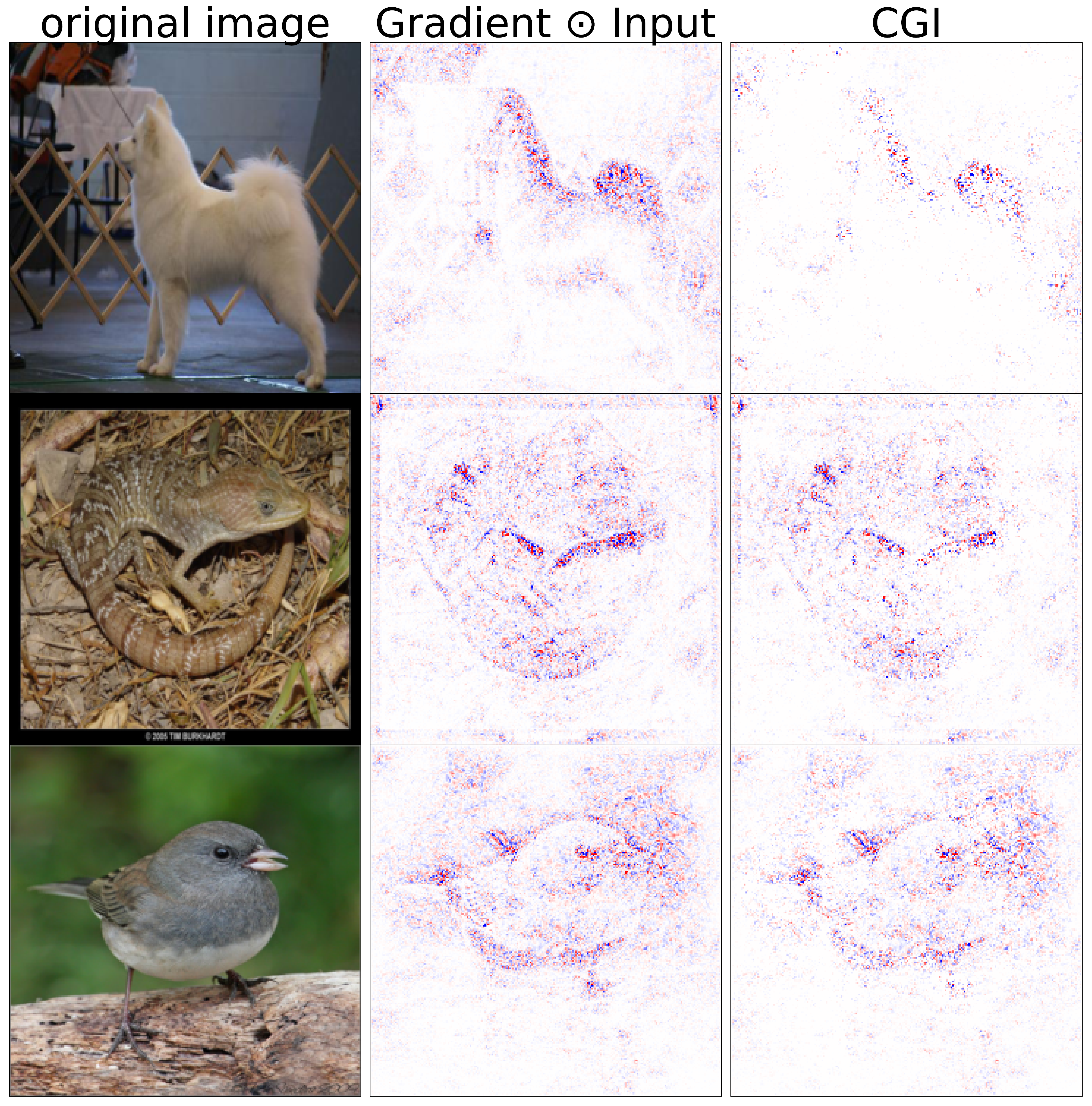}

\caption{Comparison of CGI saliency maps with Gradient $\odot$ Input saliency maps. Original images are shown on the left. }
\label{examps}
\end{figure}

Figure \ref{examps} presents an example of CGI maps on the VGG-19 architecture on Imagenet. We find that our maps are of comparable quality to Gradient $\odot$ Input.

\subsection{Parameter Randomization test}
The goal of these experiments is to determine whether CGI is sensitive to model parameters.  We run the parameter randomizaion tests on the  VGG-19 architecture  \citet{simonyan2014very} with pretrained weights on ImageNet \citet{russakovsky2015imagenet} using layerwise and cascading randomization.

 \subsubsection{Layerwise Randomization}
 
 In these experiments, we consider what happens when certain layers of the model are randomized. This represents an intermediate point between the model having learned nothing, and the model being fully trained.

 Figure \ref{divlayerabri}  shows the results of randomizing individual layers of the VGG-19 architecture with pretrained weights. (Figure  \ref{divlayer} in the Appendix shows the full figure ).The text underneath each image represents which layer of the model was randomized, with the leftmost label of 'original' representing the original saliency map of the fully trained model.  The top panel shows the saliency maps produced by {\bf CGI }, and the bottom panel the maps produces by { \bf Gradient $\odot$ Input}. We find that the { \bf Gradient $\odot$ Input} method displays the bird no matter which layer is randomized, and that our method immediately stops revealing the structure of the bird in the saliency maps as soon as any layer is randomized. Figure \ref{abslayer} in the Appendix shows a similar result but utilizing absolute value visualization. Notice that CGI's sensitivity to model parameters still holds.

\begin{figure}
\centering

\includegraphics[scale=.21]{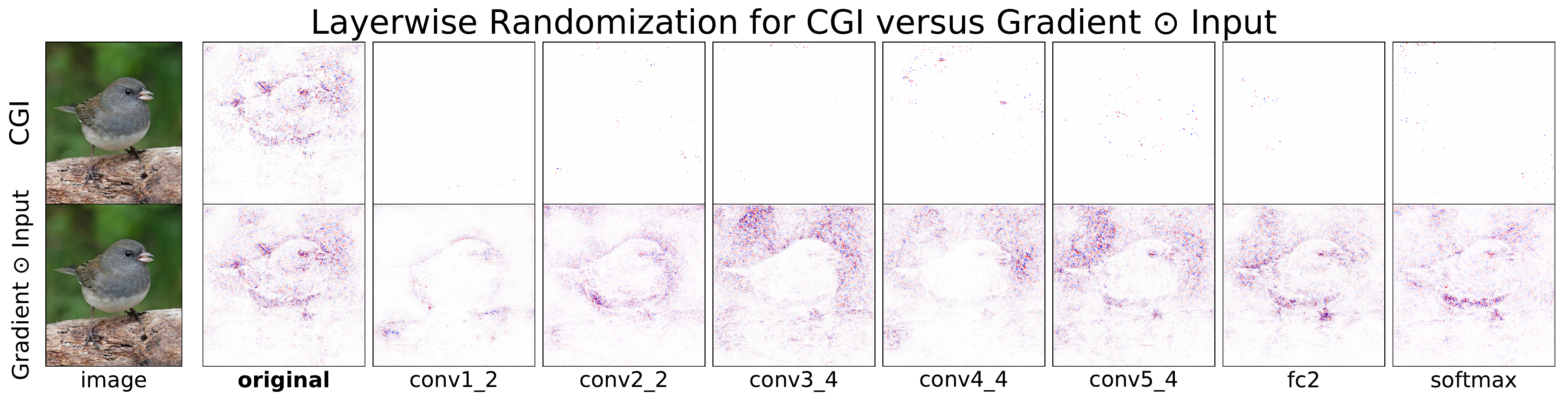}

\caption{Saliency map for layer wise randomization on VGG -19 on Imagenet  for Gradient $\odot$ Input versus CGI.  We find that in CGI, the saliency map is almost blank when any layer is reinitialized. By contrast, we find that the original Gradient $\odot$ Input method displays the structure of the bird, no matter which layer is randomized. }
\label{divlayerabri}
\end{figure}

\begin{figure}
\centering

\includegraphics[scale=.21]{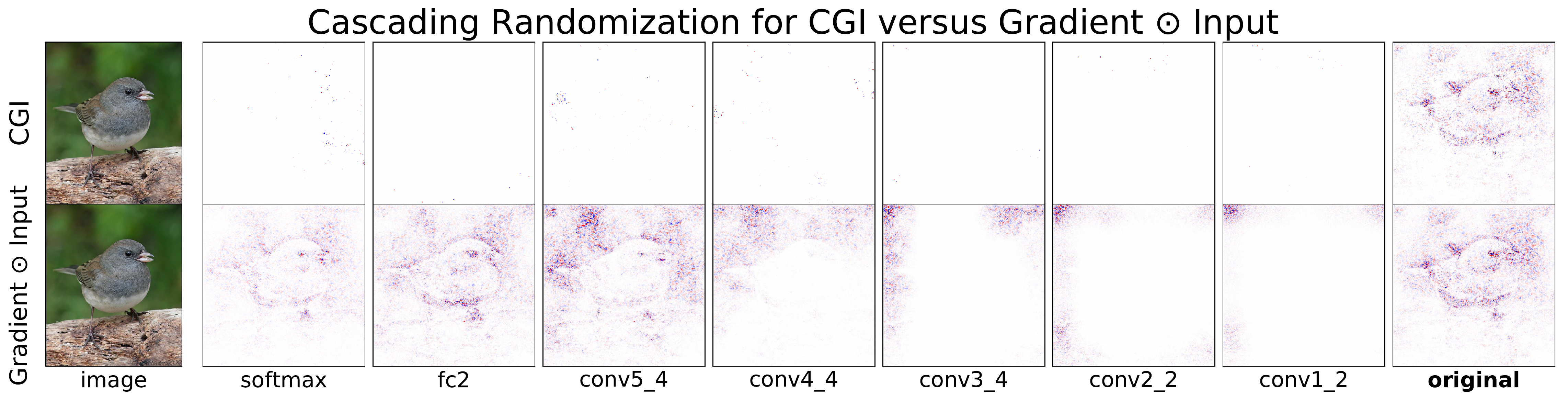}

\caption{Saliency map for cascading randomization on VGG -19 on Imagenet   for Gradient $\odot$ Input versus CGI. We find that in CGI, the saliency map is almost blank even when only the softmax layer has been reinitialized. By contrast, we find that the original Gradient $\odot$ Input method displays the structure of the bird, even after multiple blocks of randomization. }
\label{divcasc}
\end{figure}

\begin{figure}
\centering

\includegraphics[scale=.21]{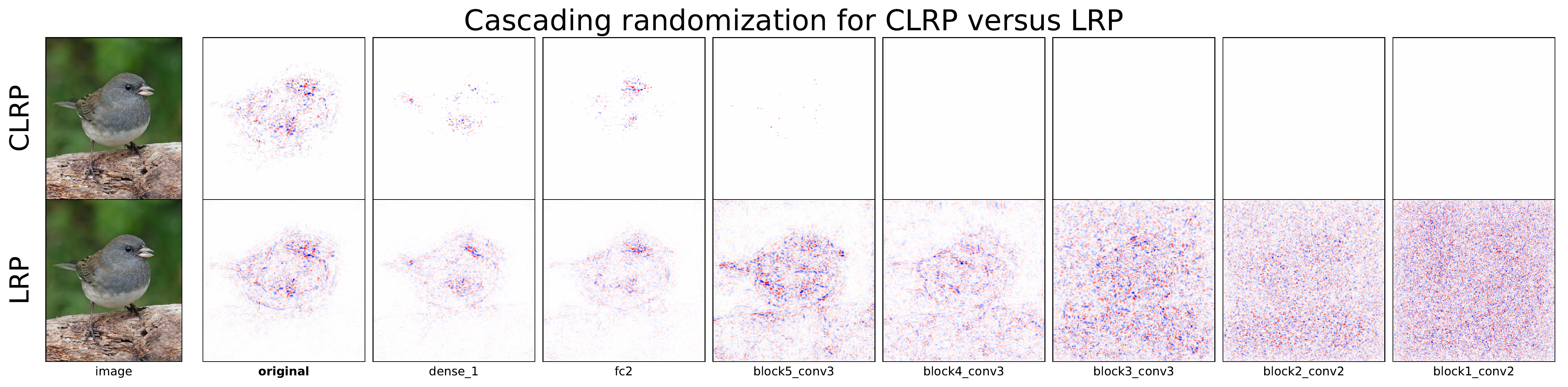}

\caption{Saliency map cascading randomization on VGG -16 on Imagenet  LRP versus CLRP. We notice that LRP shows the structure of the bird even after multiple blocks of randomization. CLRP eliminates much of the structure of the bird.  }
\label{divcasclrp}
\end{figure}

\subsubsection{Cascading Randomization}
In these experiments we consider we what happens to the saliency maps when we randomize the network weights in a cascading fashion. We randomize the weights of the VGG 19 model starting from the top layer, successively, all the way to the bottom layer.

Figure \ref{divcasc} shows our results. The rightmost figure represents the original saliency map when all layer weights and biases are set to their fully trained values. The leftmost saliency map represents the map produced when only the softmax layer has been randomized. The image to the right of that when everything up to and including conv5\_4 has been randomized, and so on.  Again we find that  { \bf CGI}  is much more sensitive to parameter randomization than  { \bf Gradient $\odot$ Input}.

\subsubsection{Comparison with LRP}

We also apply our competitive selection of pixels to LRP scores, computed using the Innvestigate library \citet{alber2018innvestigate} on the VGG-16 architecture with pretrained weights on Imagenet. The algorithm is the analogue of Algorithm \ref{alg:CGI}, but we provide the full algorithm as Algorithm \ref{alg:CLRP} in the Appendix for clarity. Figure \ref{divcasclrp} shows our results. We find that our competitive selection process (CLRP) benefits the LRP maps as well. The LRP maps show the structure of the bird even after multiple blocks of randomization, while our maps greatly reduce the prevalence of the bird structure in the images.

\subsection{Data Randomization Test}
We run experiments to determine whether our saliency method is sensitive to model training. We use a version of Alexnet \citet{krizhevsky2012imagenet} adjusted to accept one color channel instead of three and train on MNIST.  We randomly permute the lables in the training data set and train the model to greater than 98 \% accuracy and examine the saliency maps.  Figure \ref{secondsanmnist} shows our results. On the left hand side is the original image. In the middle is the map produced by {\bf Gradient \( \odot \) Input} . We find that the input structure, the number 3, still shows through with the { \bf Gradient \( \odot \) Input} method. On the other hand,  { \bf CGI}  removes the underlying structure of the number.

\label{sec:dataran}

\begin{figure}
\centering
\includegraphics[scale=0.2]{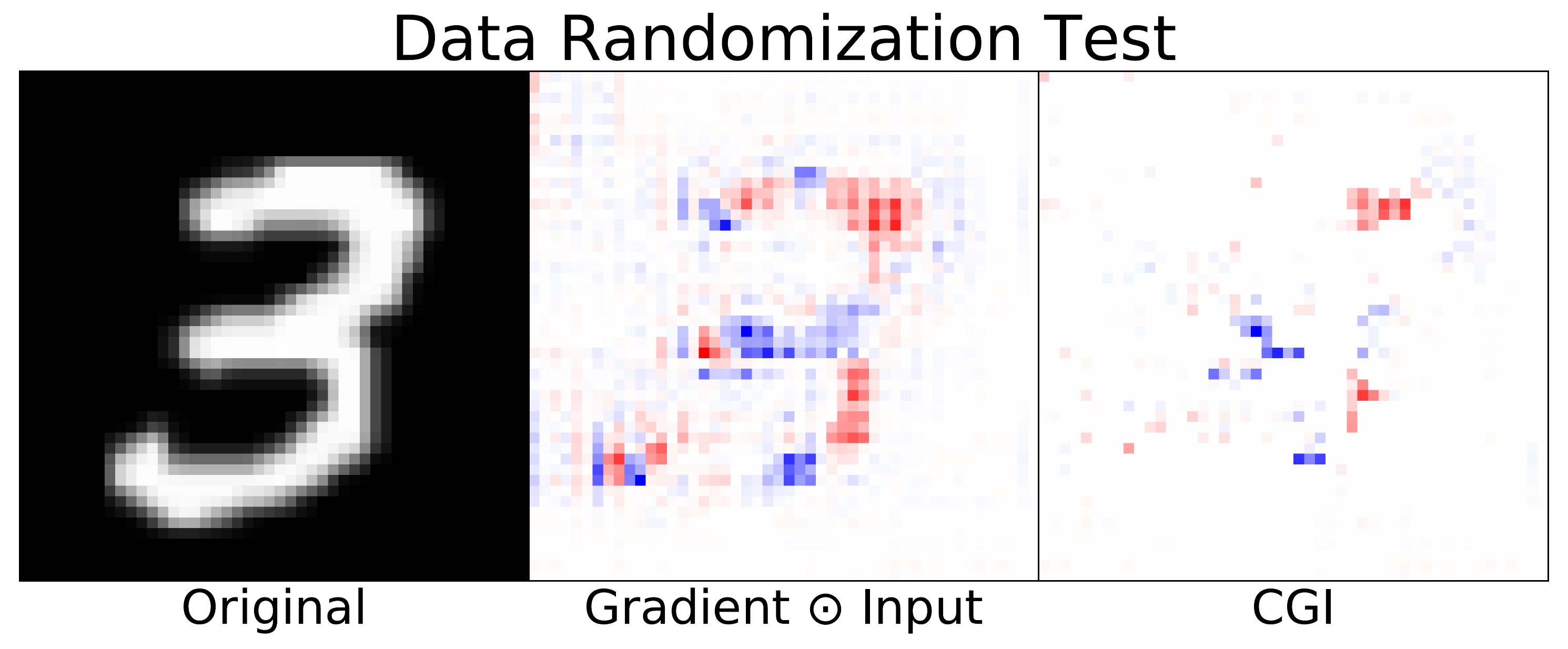}
\caption {Second sanity check for Alexnet  MNIST. On the middleimage we find that using the original gradient times input method results in an image where the original structure of the number 3 is still visible. On the right hand side image we find that our modification removes the structure of the original input image, as we would expect for a model that had been fitted on randomized data.  }
\label{secondsanmnist}
\end{figure}

%% file: conclusion.tex

\section{Conclusion}

We have introduced the idea of {\em competition among labels} as a simple modification to existing saliency methods. Unlike most past methods, this produces saliency maps by looking at the gradient of all label logits, instead of just the chosen label. Our modification keeps existing methods relevant for human evaluation (as shown on two well-known methods {\bf Gradient \( \odot \) Input} and {\bf LRP}) while allowing them to pass sanity checks of~\citet{adebayosan}, which had called into question the validity of saliency methods. Possibly our modification even improves the quality of the map, by zero-ing out irrelevant features. We gave some theory in Section~\ref{subsec:theory} to justify the competition idea for {\bf Gradient $\odot$ Input} maps for ReLU nets. 

While competition seems a good way to combine information from all logits, we leave open the question of what is the optimum way to design saliency maps by combining information from all logits\footnote{One idea that initially looked promising ---looking at gradients of outputs of the softmax layer instead of the logits---did not yield good methods in our experiments.}.

The recently-proposed sanity checks randomize the net in a significant way, either by randomizing a layer or training on corrupted data. We think it is an interesting research problem to devise  less disruptive sanity checks which are more subtle.

%% file: appendix.tex
\section{Appendix}

Let $LRP[j, Element]$ be the LRP score of Element. when decomposing output node $j$.

Let $y$ be the index of the chosen output node. 

\begin{algorithm}[H]
	\KwIn{An image $\in R^{d}$ and a neural network $S : R^{d} \rightarrow R^{C}$ }
	
	initialization: set H = $0 \in R^{d}$  vector\;
	\For{Element in Image }{
	        Calculate LRP[i,Element] for all output nodes $i$\\
		\eIf{LRP[y, Element] > 0}
		{
		
		\If{ $LRP[y, Element] \geq LRP[i, Element] \forall i \neq y$ }
		{
		Make the corresponding element of  $H$ the LRP[y, Element] }
		
		}
		{
		
		\If{ $LRP[y, Element] \leq LRP[i, Element] \forall i \neq y$}
		{
	         Make the corresponding element of  $H$ the LRP score of Element 
		}

		}
	
	}
	\caption{Competitive Layerwise Relevance Propagation}
	\label{alg:CLRP}
\end{algorithm}

\begin{figure}[ht]
\centering

\includegraphics[width=0.9\linewidth]{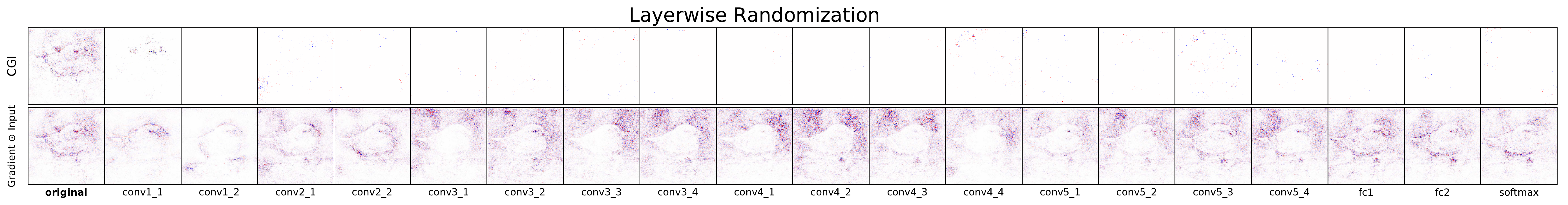}

\caption{Saliency map for layer-wise randomization of the learned weights. Diverging visualization where we plot the positive importances in red and the negative importances in blue. We find that with CGI, the saliency map is almost blank when any layer is reinitialized. By contrast, we find that  Gradient $\odot$ Input displays the structure of the bird, no matter which layer is randomized. }
\label{divlayer}
\end{figure}

\begin{figure}[ht]
\centering

\includegraphics[width=0.9\linewidth]{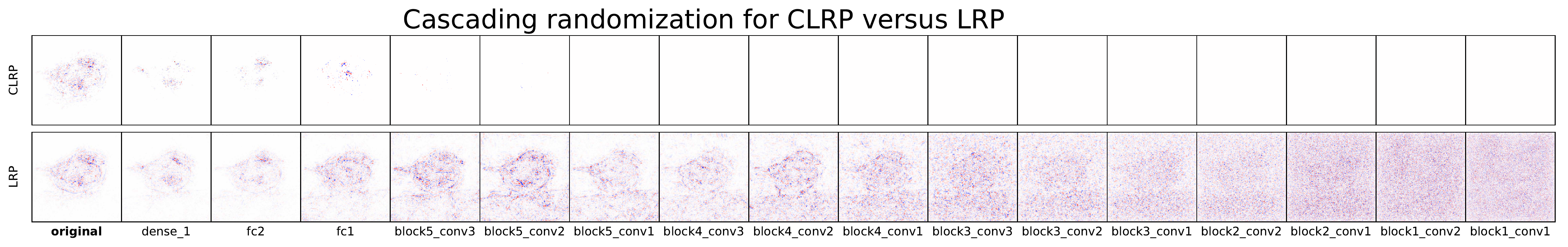}

\caption{Saliency map cascading randomization LRP versus CLRP.  }
\label{fulldivcasclrp}
\end{figure}

\begin{figure}
\centering

\includegraphics[width=0.9\linewidth]{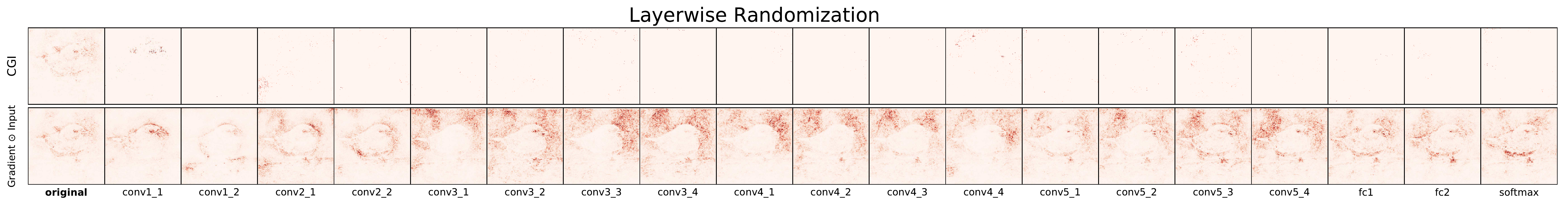}

\caption{Saliency map for layer-wise randomization of the learned weights. Absolute value visualization where we plot the absolute value of the saliency map. We find that using CGI, the saliency map is almost blank when any layer is reinitialized. By contrast, we find that Gradient $\odot$ Input displays the structure of the bird, no matter which layer is randomized. }
\label{abslayer}
\end{figure}

\begin{figure}
\centering

\includegraphics[scale=.5]{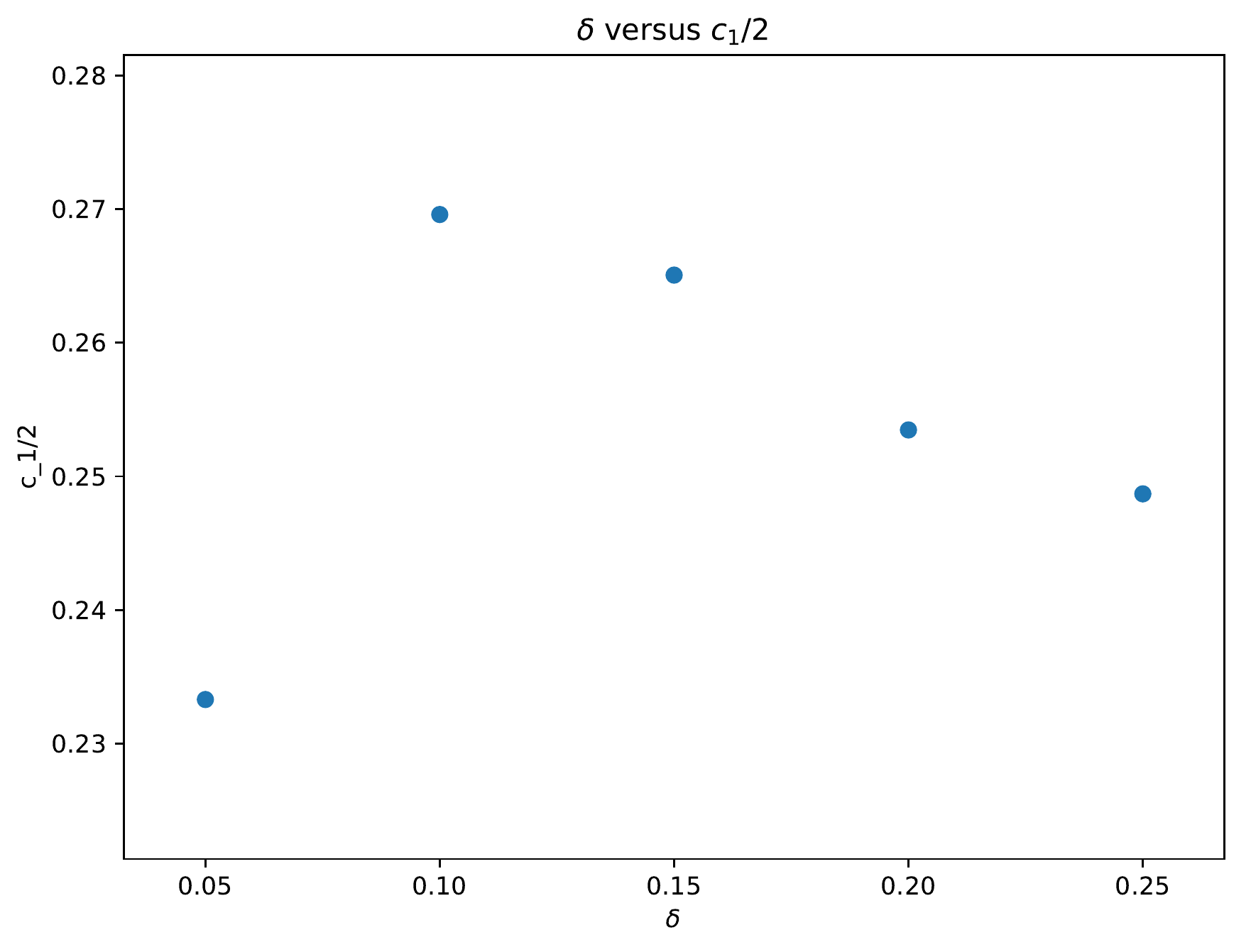}

\caption{$\delta$ versus $c_1$ for 100 averaged samples}
\label{fig:c1}
\end{figure}

\begin{figure}
\centering

\includegraphics[scale=.5]{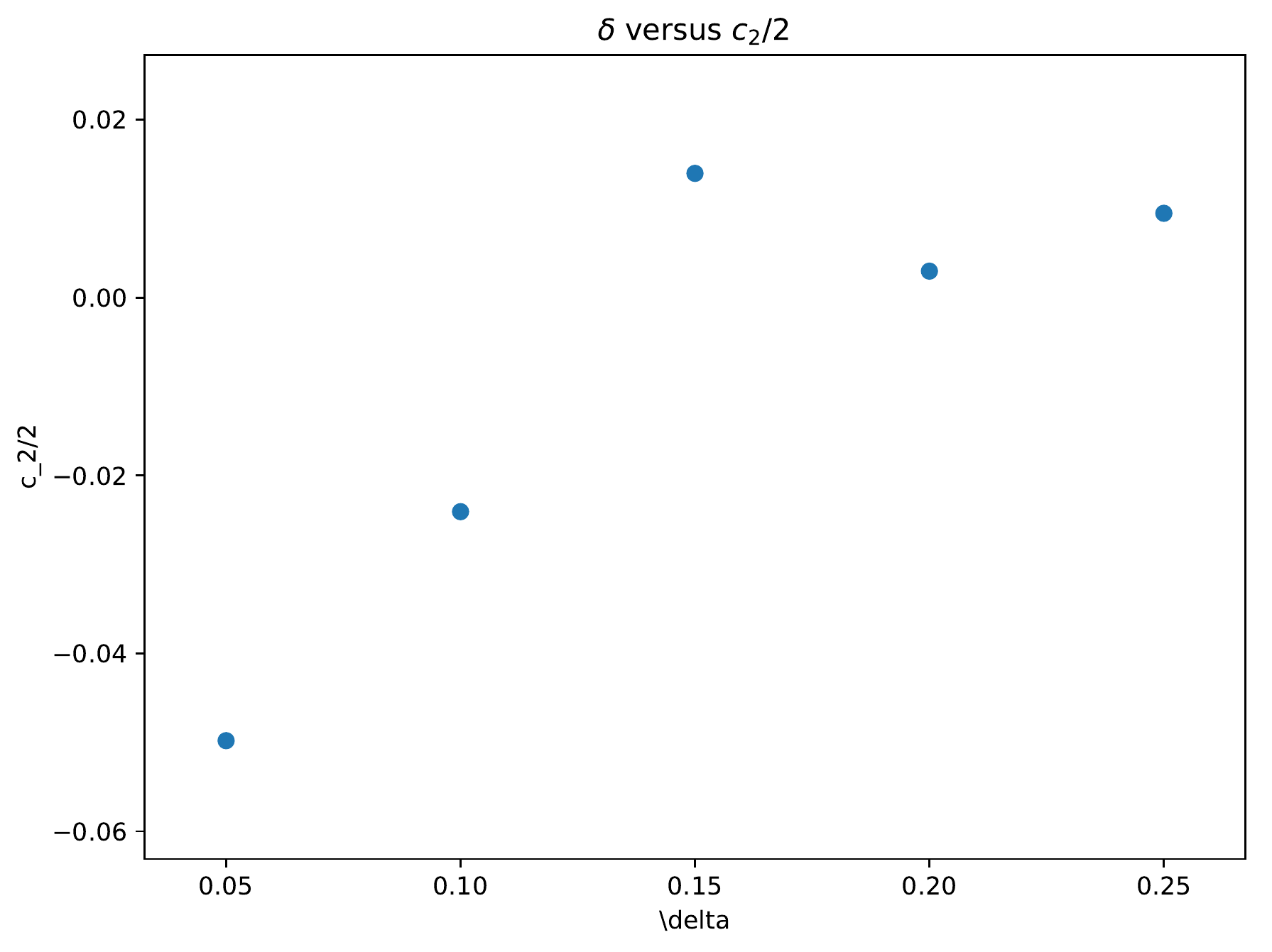}

\caption{$\delta$ versus $c_2$ for 100 averaged samples}
\label{fig:c2}
\end{figure}